\documentclass[twoside]{article}
\usepackage[accepted]{aistats2018}
\usepackage{amsmath}
\usepackage{amssymb}
\usepackage{amsthm}
\usepackage{bm}
\usepackage{times}
\usepackage{graphicx} 
\usepackage{subfigure} 
\usepackage{multirow}
\usepackage{color}
\usepackage{natbib}
\usepackage{algorithmicx,algorithm}
\usepackage[noend]{algpseudocode}
\usepackage{multirow}
\usepackage{url}

\newtheorem{theorem}{Theorem}[section]
\newtheorem{lemma}[theorem]{Lemma}

\begin{document}

%

%

\twocolumn[

\aistatstitle{Benefits from Superposed Hawkes Processes}

\aistatsauthor{ Hongteng Xu$^{1,2}$ \And Dixin Luo$^1$ \And  Xu Chen$^3$ \And Lawrence Carin$^1$}

\aistatsaddress{ $^1$Department of ECE, Duke University \And  $^2$InfiniaML, Inc. \And $^3$School of Software, Tsinghua University} 

]

\begin{abstract}
The superposition of temporal point processes has been studied for many years, although the usefulness of such models for practical applications has not be fully developed. 
We investigate superposed Hawkes process as an important class of such models, with properties studied in the framework of least squares estimation. 
The superposition of Hawkes processes is demonstrated to be beneficial for tightening the upper bound of excess risk under certain conditions, and we show the feasibility of the benefit in typical situations. 
The usefulness of superposed Hawkes processes is verified on synthetic data, and its potential to solve the cold-start problem of recommendation systems is demonstrated on real-world data.
\end{abstract}

\section{Introduction}
Given a set of temporal point processes $\{N^m\}_{m=1}^{M}$, their superposition is a new point process $N$ defined by the sum of counting processes, $i.e.$, $N(t) = \sum_{m=1}^{M}N^m(t)$, $t\geq 0$. 
For the superposed point process, its instantiated event sequence is the superposition of the event sequences corresponding to $\{N^m\}_{m=1}^{M}$. 
The study of superposed point process has a long history, and many interesting properties have been found~\citep{cox1954superposition,cinlar1968superposition,albin1984approximating}. 
However, there exists a marked gap between the study of superposed point processes and practical applications. 
Existing work mainly focuses on the superposition of simple point processes, $e.g.$, Poisson processes~\citep{cinlar1968superposition} and renewal processes~\citep{cox1954superposition}. 
These models are oversimplified to describe the mechanism of real-world event sequences, while the properties of the superposition of more complicated point processes are not fully investigated. 
More essentially, \emph{can we get any benefits from learning superposed point processes?} \emph{If we can, what would the benefits be in practice?}
These are still significant open problems.

Focusing on an important class of point process models, called the Hawkes process~\citep{hawkes1971point}, we give positive answers for the above questions. 
In particular, we prove that there indeed exist benefits from superposed Hawkes processes in the framework of least-squares learning. 
Theoretically, for the Hawkes processes with different \emph{exogenous} intensities and shared \emph{endogenous} triggering patterns, we can learn the parameters of the endogenous terms with a tighter bound on excess risk, by superposing the Hawkes processes together. 
We analyze this superposition-based learning in depth and quantitatively connect its feasibility with the diversity of exogenous intensity. 
Moreover, we validate the benefits from superposed Hawkes processes on both synthetic and real-world data. 
We show that learning superposed Hawkes processes is beneficial to solve the cold-start problem of recommendation systems.

\section{Related work}
\subsection{Superposed point processes}
As aforementioned, the superposition of temporal point processes has been studied for decades. 
The early work in~\citep{cox1954superposition} studied the superposition of renewal processes and applied its property to analyze pooling signals in neurophysiology. 
The work in~\citep{cinlar1968superposition} analyzed the independence of source processes and the dynamics of the source indicator, for the superposition of Poisson processes and that of renewal processes. 
This work is further extended to multi-dimensional point processes in~\citep{ccinlar1968superposition}. 
Recently, the work in~\citep{moller2012transforming} proved that a spatial point process can be transformed to a Poisson process by superposing its observations randomly. 
From the viewpoint of applications, the superposition of arrival processes is applied to model queue behaviors, which can be learned as a renewal process~\citep{albin1984approximating}. 
Additionally, the superposition of arrival processes is used to analyze voice data in~\citep{sriram1986characterizing}. 
More recently, Bayesian-based methods are proposed for classifying source processes from superposed observations~\citep{walsh2005classification}, and their learning algorithms can be implemented based on MCMC~\citep{redenbach2015classification} or variational inference~\citep{rajala2016variational}.
However, all of these research fruits are based on simple point processes, like Poisson and renewal processes. 
The superposition of more complicated point processes, $e.g.$, the superposition of (multi-dimensional) Hawkes processes, has not been investigated yet. 
What is worse, the previous work above always treats superposed point processes as ``challenges'' in statistical analysis and practical applications. 
None of them consider potential ``benefits'' from superposed point processes. 
Our work fills this gap from the viewpoint of learning Hawkes processes.

\subsection{Hawkes processes}
Hawkes processes~\citep{hawkes1971point,hawkes1974cluster} are useful tools for modeling and analyzing the mutual-excitation phenomena commonly observed in real-world event sequences, which can be applied to many problems, such as social network analysis~\citep{zhao2015seismic,wang2017linking} and quantitative finance~\citep{bacry2012non,hardiman2013critical}. 
Many variants of Hawkes processes have been proposed recently, $e.g.$, the mixture of Hawkes processes~\citep{yang2013mixture,xu2017dirichlet}, the nonlinear isotonic Hawkes process~\citep{wang2016isotonic} and the time-varying Hawkes process~\citep{xu17b}, which show potential to analyze complicated event sequences. 
From the perspective of learning methodology, maximum likelihood estimation is one of the most popular approaches for learning Hawkes processes~\citep{lewis2011nonparametric,zhou2013learning}. 
Recently, least-squares-based learning methods~\citep{eichler2017graphical}, Wiener-Hopf-based methods~\citep{bacry2012non} and the cumulants-based methods~\citep{achab2016uncovering} are also used to learn and analyze Hawkes processes. 
However, these methods do not consider the influence of superposition on learning.

\section{Superposed Hawkes Processes}
\subsection{Learning Hawkes processes as linear predictors}\label{ssec2:ls}
Consider $D$ entities with interactions ($e.g.$, $D$ users in a social network). 
For each entity $d\in\mathcal{D}$, $\mathcal{D}=\{1,...,D\}$, we observe its behaviors at timestamps $\{t_{d,1},t_{d,2},...\}$, which are represented by a counting process $N_d(t)=|\{t_{d,i}|t_{d,i}\leq t,~i=1,2,...\}|$, $i.e.$, the number of type-$d$ events before and at time $t$. 
Here $|\cdot|$ represent the cardinality of a set, \textcolor{black}{and $i$ indicates the index of event in each observed sequence.} 
Accordingly, the sequence of all entities' behaviors, denoted as $\{(t_i, d_i)\}$, can be represented by a $D$-dimensional counting process $N(t)=[N_d(t)]\in\mathbb{N}^D$.
For each $N_d(t)$, the expected instantaneous happening rate of an event is represented by its intensity function:
\begin{eqnarray}\label{intensity}
\begin{aligned}
\lambda_d(t) = \frac{\mathbb{E}[dN_d(t)|\mathcal{H}_t]}{dt},
\end{aligned}
\end{eqnarray}
where $\mathcal{H}_t$ contains all historical events happening before or at time $t$. 

The $D$-dimensional counting process may be modeled by a $D$-dimensional Hawkes process, and the intensity function has the form:
\begin{eqnarray}\label{hp}
\begin{aligned}
\lambda_d(t) &= \mu_d(t) + \sideset{}{_{d'=1}^{D}}\sum\int_{0}^{T}\phi_{dd'}(t,~s)dN_{d'}(s)\\
&=\mu_d(t) + \sideset{}{_{(t_i,d_i)\in\mathcal{H}_t}}\sum\phi_{dd_i}(t,~t_i).
\end{aligned}
\end{eqnarray}
where $\bm{\mu}(t)=[\mu_d(t)]$ corresponds to the background intensity caused by some exogenous factors. 
Generally, we can model $\bm{\mu}(t)$ as a $D$-dimensional homogeneous or inhomogeneous Poisson process. 
The term $\sum_{d'=1}^{D}\int_{0}^{T}\phi_{dd'}(t,s)dN_{d'}(s)$ represents the accumulation of endogenous intensity caused by history~\citep{farajtabar2014shaping}. 
The impact function (or link function) $\phi_{dd'}(t,s)$, $t\geq s$, represents the influence of the $d'$-th entity on the $d$-th entity when their corresponding behaviors (or events) happen at time $s$ and $t$, respectively. 
We often assume that the target Hawkes process is shift-invariant: $\phi_{dd'}(t,s)=\phi_{dd'}(t-s)$. 
For convenience, we represent a Hawkes process as $HP(\bm{\mu},\bm{\Phi})$, where $\bm{\Phi}(t)=[\phi_{dd'}(t)]$, and its instantiated counting process is $N(t)\sim HP(\bm{\mu},\bm{\Phi})$. 

We may often model $\bm{\mu}(t)$ as a $D$-dimensional vector $\bm{\mu}=[\mu_d]$, and $\phi_{dd'}(t)$ as $a_{dd'}\kappa(t)$~\citep{zhou2013learning}, where $\kappa(t)$ is a predefined decay kernel like exponential kernel, $i.e.$, $\kappa(t)=\exp(-wt)$. 
This model means that the exogenous intensity is a $D$-dimensional homogeneous Poisson process, while the impact functions $\bm{\Phi}(t)$ can be parameterized by an infectivity matrix $\bm{A}=[a_{dd'}]$. 
In such a situation, the Hawkes process corresponds to a parametric model with $\bm{\theta}=[\bm{\mu}; \mbox{vec}(\bm{A})]\in\mathbb{R}^{D(1+D)}$, where $\mbox{vec}(\cdot)$ vectorizes its input. 
Accordingly, its intensity function can be represented as a linear function of $\bm{\theta}$:
\begin{eqnarray}
\begin{aligned}
\lambda_d(t) = \bm{x}_d^{\top}(t)\bm{\theta},
\end{aligned}
\end{eqnarray}
where $\bm{x}_d(t)=[\bm{e}_{d};\mbox{vec}(\bm{E}(t))]$.
$\bm{e}_{d}\in\mathbb{R}^{D}$, whose elements are zeros except the $d$-th one, which has value $1$, corresponds to $\mu_d$. 
$\bm{E}(t)=[e_{dd'}(t)]\in\mathbb{R}^{D\times D}$, where $e_{dd'}(t)=\sum_{(t_i, d_i)\in\mathcal{H}_t,~d_i=d'}\kappa(t-t_i)$, which corresponds to the accumulated decay kernels caused by historical events.

To learn the parameters of the model, we may minimize the squared loss between the counting processes of instantiated event sequences and the integration of intensity function, as in~\citep{eichler2017graphical,wang2016isotonic}. 
Specifically, given $M$ event sequences with $I$ events per each, the squared loss is
\begin{eqnarray}\label{sqloss}
\begin{aligned}
&\mathbb{E}\Bigl[\Bigl(N(t) - \int_{0}^{t}\lambda(s)ds\Bigr)^{2}\Bigr]\\
&=\frac{1}{MI}\sideset{}{_{m=1}^{M}}\sum\sideset{}{_{i=1}^{I}}\sum\Bigl|N^m_{d_i^m}(t_i^m) - \int_{0}^{t_i^m}\lambda_{d_i^m}(s)ds\Bigr|^2\\
&=\|\bm{N}-\bm{X}\bm{\theta}\|_2^2.
\end{aligned}
\end{eqnarray} 
Here, $\bm{N}=\frac{1}{\sqrt{MI}}[\bm{N}^1; ...;\bm{N}^M]\in\mathbb{R}^{MI}$ represents all observed counting processes, in which each $\bm{N}^m=[N^m_{d_1^m}(t_1^m);...;N^m_{d_I^m}(t_I^m)]$ contains the number of events with specific types till observed times stamps. 
Similarly, $\bm{X}=\frac{1}{\sqrt{MI}}[\bm{X}^1;...;\bm{X}^M]\in\mathbb{R}^{MI\times D(1+D)}$, and each $\bm{X}^m=[\int_{0}^{t_1^m}\bm{x}_{d_1^m}^{\top}(s)ds;...;\int_{0}^{t_I^m}\bm{x}_{d_I^m}^{\top}(s)ds]\in \mathbb{R}^{I\times D(1+D)}$, where the integration is an element-wise operation. 

\textcolor{black}{
It should be noted that in practice it is unnecessary for each sequence to have the same number of events. 
However, without loss of generality, we assume each event sequence has $I$ events in the following theoretical content for convenience.
} 

From the viewpoint of machine learning, (\ref{sqloss}) measures the risk that the observed counting processes are different from their expectations estimated by a linear predictor, where $\bm{N}$ and $\bm{X}$ are labels and samples, respectively. 
An ideal predictor can obtain the real expectation, such that (\ref{sqloss}) corresponds to the variance of counting processes. 
From the analysis in~\citep{bacry2012non,hardiman2013critical}, the variance $\mathbb{E}[(N(t)-\int_0^t\lambda(s)ds)^2]\sim \mathcal{O}(t^2)$. 
Therefore, differing from a traditional least squares-based linear predictor, the residual errors between sample-label pair ($i.e.$, $N_{d_i^m}(t_i^m) - (\int_{0}^{t_i^m}\bm{x}_{d_i^m}^{\top}(s)ds)\bm{\theta}$) at different time stamps do not obey the same Gaussian distribution. 
To solve this problem, we rescale the labels and samples according to the property of the variance mentioned above, and obtain a weighted squared loss (the risk of proposed linear predictor) as:
\begin{eqnarray}\label{sqloss_s}
\begin{aligned}
R_{single}(\bm{\theta})=\|\bm{W}(\bm{N}-\bm{X}\bm{\theta})\|_2^2,
\end{aligned}
\end{eqnarray}
where the diagonal matrix $\bm{W}=\mbox{diag}(\bm{W}^1,...,\bm{W}^M)$ and each $\bm{W}^m=\mbox{diag}(\frac{1}{t_1^m},...\frac{1}{t_I^m})$. 
Here we denote the loss as $R_{single}$ because it corresponds to learning a single Hawkes process from $M$ observations.

In many practical situations, the systems described by Hawkes processes are endogenously stationary, and their fluctuations are caused by the changes of exogenous intensities. 
For example, the interactions between users ($i.e.$, entities) in social networks can be modeled by Hawkes processes. 
In practice, the event sequences we observed may share the same endogenous triggering patterns, while having different exogenous intensities, because the infectivity among different users ($i.e.$, the impact functions $\bm{\Phi}(t)$ or their infectivity matrix $\bm{A}$) is stationary in a long time range~\citep{zhou2013learning}, but these sequences of users' behaviors may be driven by different information sources and contents~\citep{farajtabar2014shaping}. 

In the case of multiple sources, we require $M$ Hawkes processes to model the $M$ event sequences. 
The processes have individual $\{\bm{\mu}^m\}_{m=1}^{M}$ and shared infectivity $\bm{A}$. 
We can still learn these models jointly by solving a least squares problem, in which the parameter $\bm{\theta}_{multi}=[\bm{\mu}^1;...;\bm{\mu}^M;\mbox{vec}(\bm{A})]\in\mathbb{R}^{D(M+D)}$ and the squared loss is 
\begin{eqnarray}\label{sqloss_m}
\begin{aligned}
R_{multi}(\bm{\theta}_{multi})=\|\bm{W}(\bm{N}-\bm{X}_{multi}\bm{\theta}_{multi})\|_2^2,
\end{aligned}
\end{eqnarray}
where $\bm{X}_{multi}=[\bm{X}_{\mu}, \bm{X}_{A}]\in\mathbb{R}^{MI\times D(M+D)}$. 
The last $D^2$ columns of $\bm{X}_{multi}$ ($i.e.$, $\bm{X}_A$) are the same with those of $\bm{X}$ in (\ref{sqloss}), while the first $MD$ columns ($i.e.$, $\bm{X}_{\mu}$) are sparse, which corresponds to different $\bm{\mu}^m$. 
Specifically, for the sample $N_{d_i^m}(t_i^m)$, the $(I(m-1)+i)$-th row of $\bm{X}_{\mu}$ are all zeros except the $(D(m-1)+d_i^m)$-th element, with value $\frac{1}{\sqrt{MI}}$.   

Equation (\ref{sqloss_m}) may be viewed as a special case of multi-task learning of Hawkes processes in~\citep{luo2015multi}, in which Hawkes processes with different exogenous intensities share the same endogenous triggering patterns. 
If the exogenous intensities have certain structures, $e.g.$, the $\bm{\mu}^m$'s are sparse or they are grouped and low-rank, we can further impose some regularizers in (\ref{sqloss_m}). 
However, learning multiple Hawkes processes jointly ($i.e.$, $\min_{\bm{\theta}_{multi}} R_{multi}$) is harder than learning a single one ($i.e.$, $\min_{\bm{\theta}_{single}} R_{single}$), which has more parameters and requires more observations. 
In the following subsection we will show that by superposing the Hawkes processes, we can obtain better learning results with fewer observations, especially for the endogenous impact functions and the corresponding infectivity matrix. 

\subsection{Benefits from superposition} 
For independent Hawkes processes with shared impact functions, their superposition has the following property:
\begin{theorem}\label{thm1}
\textcolor{black}{For $M$ independent Hawkes processes with shared impact functions, where $N^m(t)\sim HP(\bm{\mu}^m(t),\bm{\Phi})$ and $m=1,...,M$, their superposition is still a Hawkes process, $i.e.$, $N(t)=\sum_{m=1}^{M}N^m(t)$ and $N(t)\sim HP(\sum_{m=1}^{M}\bm{\mu}^m(t),\bm{\Phi})$.}
\end{theorem}
\begin{proof}
Because $N(t)=\sum_{m=1}^{M}N^m(t)$, for $d\in\mathcal{D}$, its intensity is
\begin{eqnarray*}
\begin{aligned}
\lambda_d(t)=&\frac{\mathbb{E}[dN_d(t)|\mathcal{H}_t]}{dt}
=\sideset{}{_{m=1}^{M}}\sum\frac{\mathbb{E}[dN_d^m(t)|\cup_{l=1}^{M}\mathcal{H}_t^l]}{dt}\\
=&\sideset{}{_{m=1}^{M}}\sum\frac{\mathbb{E}[dN_d^m(t)|\mathcal{H}_t^m]}{dt}
=\sideset{}{_{m=1}^{M}}\sum\lambda_d^m(t).
\end{aligned}
\end{eqnarray*}
Here $\mathcal{H}_t=\cup_{m=1}^{M}\mathcal{H}_t^{m}$ contains all historical events in the superposed process. 
For the Hawkes processes with shared impact functions, we have
\begin{eqnarray*}
\begin{aligned}
\lambda_d(t) =& \sideset{}{_{m=1}^{M}}\sum\Bigl(\mu_d^m(t) +\sideset{}{_{(t_i^m, d_i^m)\in\mathcal{H}_t^{m}}}\sum\phi_{dd_i^m}(t-t_i^m)\Bigr)\\
=&\sideset{}{_{m=1}^{M}}\sum\mu_d^m(t) + \sideset{}{_{(t_i, d_i)\in\mathcal{H}_t}}\sum\phi_{dd_i}(t-t_i),
\end{aligned}
\end{eqnarray*}
for $d\in\mathcal{D}$. 
According to the definition of Hawkes process in (\ref{hp}), we have $N(t)\sim HP(\sum_{m=1}^{M}\bm{\mu}^m(t),\bm{\Phi})$.
\end{proof}

This property implies that when we aim to learn the impact functions of the target Hawkes processes, besides learning from independent observations, we can learn from the superposed observations of the Hawkes processes. 
In particular, given $\{N^m(t)\}_{m=1}^{M}$, the aforementioned traditional strategy learns multiple Hawkes processes jointly by $\min_{\bm{\theta}_{multi}}R_{multi}(\bm{\theta}_{multi})$. 
The optimal solution $\hat{\bm{\theta}}_{multi}=[\hat{\bm{\mu}}^{1};...;\hat{\bm{\mu}}^{M};\mbox{vec}(\widehat{\bm{A}})]$. 
By contrast, our strategy first obtains a superposed Hawkes process $N_t=\sum_m N^m(t)$, and then learns a single Hawkes process by $\min_{\bm{\theta}_{super}}R_{super}(\bm{\theta}_{super})$, where 
\begin{eqnarray}\label{sqloss_super}
\begin{aligned}
R_{super}(\bm{\theta}_{super})&
=\frac{1}{M^2}R_{single}(\bm{\theta}_{super})\\
&=\Bigl\|\frac{1}{M}\bm{W}(\bm{N}-\bm{X}\bm{\theta}_{super})\Bigr\|_2^2.
\end{aligned}
\end{eqnarray}
The scaling constant $\frac{1}{M}$ in the last term ensures that the dynamic range of the superposed counting process $N(t)$ is approximately the same as that of a single counting process $N^m(t)$, $m=1,...,M$. 
Its optimal solution $\hat{\bm{\theta}}_{super}=[\sum_{m=1}^{M}\hat{\bm{\mu}}^{m};\mbox{vec}(\widehat{\bm{A}})]$. 
Given $\widehat{\bm{A}}$, we can further estimate $\bm{\mu}^m$ by solving $M$ independent least squares problem.

Although our superposition-based strategy cannot learn exogenous intensities simultaneously with endogenous impact functions, it provides benefits for learning impact functions. 
Specifically, given observed samples we can obtain a tighter bound on the excess risk under a certain condition:
\begin{theorem}\label{thm2}
Suppose that we have $M$ independent and stationary $D$-dimensional Hawkes processes with shared impact functions, $i.e.$, $\{HP(\bm{\mu}^m, \bm{A})\}_{m=1}^{M}$, where the parameters are bounded as $\|\bm{\mu}^m\|_2^2\leq B_{\mu}$ and $\|\mbox{vec}(\bm{A})\|_2^2\leq B_{A}$. 
Each of them has an observed event sequence with $I$ events. 
Then the bound on the excess risk $\mathbb{E}[R_{super}(\hat{\bm{\theta}}_{super})-R_{super}(\bm{\theta}_{super}^*)]$ is tighter than that of $\mathbb{E}[R_{multi}(\hat{\bm{\theta}}_{multi})-R_{multi}(\bm{\theta}_{multi}^*)]$ when the upper bound of $\|\sum_{m=1}^{M}\bm{\mu}^m\|_2^2$, denoted as $B_{\Sigma\mu}$, satisfies
\begin{eqnarray}\label{cond}
\begin{aligned}
B_{\Sigma\mu}\leq & MB_{\mu} + D(M+D)B_{\mu}\log\Bigl(1+\frac{MI}{D(M+D)}\Bigr) \\
&- D(1+D)B_{\mu}\log\Bigl(1+\frac{MI}{D(1+D)}\Bigr).
\end{aligned}
\end{eqnarray}
Here $\bm{\theta}^*$ represents the ground truth of parameters.
\end{theorem} 
\begin{proof}
The heart of the proof is the upper bound on the excess risk of linear predictor derived by~\citep{shamir2015sample}. 
In particular, for the linear predictor $\hat{\bm{\theta}}$ learned by minimizing the squared loss $R(\bm{\theta})=\|\bm{y} - \bm{X\theta}\|_2^2$, where $\bm{\theta}\in \{\bm{\theta}\in\mathbb{R}^C:\|\bm{\theta}\|_2^2\leq B\}$ and the $M$ observations $\bm{y}=[y_1;...;y_M]$ satisfy $y_i\in\{y:|y|\leq Y\}$, we have
\begin{eqnarray}\label{upper}
\begin{aligned}
\mathbb{E}[R(\hat{\bm{\theta}})-R(\bm{\theta}^*)]\leq\mathcal{O}\Bigl(\frac{B+CY^2\log(1+\frac{M}{C})}{M}\Bigr).
\end{aligned}
\end{eqnarray}
For the loss functions in~(\ref{sqloss_m},~\ref{sqloss_super}), the observations $\bm{N}$ is re-scaled by $\frac{1}{M}\bm{W}$ and $\bm{W}$, respectively. 
Additionally, the analysis in~\citep{zhu2013ruin} shows that $\lim_{t\rightarrow \infty}\frac{N_d(t)}{t}=\frac{\mu_d}{1-\|\bm{\Phi}\|}$ for $d=1,...,D$. 
Because of the stationarity of Hawkes processes, we have $1-\|\bm{\Phi}\|\gg 0$, and accordingly, the range of the scaled observations should be in the same order of magnitude with the exogenous intensities. 
Therefore, the $Y^2$ in (\ref{upper}) can be replaced by $\mathcal{O}(B_{\mu})$ in our work. 

According to the analysis above, we apply (\ref{upper}) to the two learning strategies mentioned above and obtain
\begin{eqnarray*}
\begin{aligned}
&\mathbb{E}[R_{multi}(\hat{\bm{\theta}}_{multi})-R_{multi}(\bm{\theta}_{multi}^*)]\\
&\leq \mathcal{O}\Bigl( \frac{B_A+MB_{\mu}+D(M+D)B_{\mu}\log(1+\frac{MI}{D(M+D)})}{MI} \Bigr),\\
&\mathbb{E}[R_{super}(\hat{\bm{\theta}}_{super})-R_{super}(\bm{\theta}_{super}^*)]\\
&\leq \mathcal{O}\Bigl( \frac{B_A+B_{\Sigma\mu}+D(1+D)B_{\mu}\log(1+\frac{MI}{D(1+D)})}{MI} \Bigr).
\end{aligned}
\end{eqnarray*}
Here, both of these two strategies use $MI$ samples in their loss functions. 
However, the dimension of $\bm{\theta}_{multi}$ is $D(M+D)$ for learning multiple Hawkes processes while the dimension of $\bm{\theta}_{super}$ is just $D(1+D)$ for learning a superposed Hawkes process. 
$B_A+MB_{\mu}$ is the upper bound of $\|\bm{\theta}_{multi}\|_2^2$ and $B_A+B_{\Sigma\mu}$ is the upper bound of $\|\bm{\theta}_{super}\|_2^2$. 
Clearly, $\mathbb{E}[R_{super}(\hat{\bm{\theta}}_{super})-R_{super}(\bm{\theta}_{super}^*)]\leq \mathbb{E}[R_{multi}(\hat{\bm{\theta}}_{multi})-R_{multi}(\bm{\theta}_{multi}^*)]$ when
\begin{eqnarray}
\begin{aligned}
&B_{\Sigma\mu}+D(1+D)B_{\mu}\log(1+\frac{MI}{D(1+D)})\\
&\leq MB_{\mu}+D(M+D)B_{\mu}\log(1+\frac{MI}{D(M+D)}).
\end{aligned}
\end{eqnarray}
This completes the proof.
\end{proof}

Theorem~\ref{thm2} means that under a certain condition learning a superposed Hawkes process can get better convergence of the loss function with fewer samples compared with learning multiple Hawkes processes, which is meaningful to improve the robustness of learning endogenous impact functions. 
Note that this benefit from superposition is only available when the condition in (\ref{cond}) is satisfied. 
Based on Theorem~\ref{thm2}, we can show that the diversity of the exogenous intensities has a large influence on the feasibility of our superposition-based learning strategy. 
In particular, we have the following two lemmas.
\begin{lemma}\label{lm1}
For the Hawkes processes with the same exogenous intensity and endogenous impact functions, the superposition-based learning strategy is inefficient, $i.e.$, $\mathbb{E}[R_{super}(\hat{\bm{\theta}}_{super})-R_{super}(\bm{\theta}_{super}^*)]\geq \mathbb{E}[R_{single}(\hat{\bm{\theta}})-R_{single}(\bm{\theta}^*)]$.
\end{lemma}
\begin{proof}
In this case, $\bm{\mu}^1=...=\bm{\mu}^M=\bm{\mu}$ and $\|\sum_{m=1}^{M}\bm{\mu}^m\|_2^2=M^2\|\bm{\mu}\|_2^2\leq M^2B_{\mu}=B_{\Sigma\mu}$. 
Instead of learning multiple Hawkes processes, we only need to learn a single Hawkes process from multiple independent event sequences or from a superposition of them. 
Similar to the proof of Theorem~\ref{thm2}, when we minimize $R_{single}(\bm{\theta})$, its excess risk is bounded as
\begin{eqnarray*}
\begin{aligned}
&\mathbb{E}[R_{single}(\hat{\bm{\theta}})-R_{single}(\bm{\theta}^*)]\\
&\leq \mathcal{O}\Bigl( \frac{B_A+B_{\mu}+D(1+D)B_{\mu}\log(1+\frac{MI}{D(1+D)})}{MI} \Bigr).
\end{aligned}
\end{eqnarray*}
Based on the relationship that $B_{\Sigma\mu}=M^2B_{\mu}$, when we minimize $R_{super}(\bm{\theta}_{super})$, its excess risk is bounded as
\begin{eqnarray*}
\begin{aligned}
&\mathbb{E}[R_{super}(\hat{\bm{\theta}}_{super})-R_{super}(\bm{\theta}_{super}^*)]\\
&\leq \mathcal{O}\Bigl( \frac{B_A+M^2B_{\mu}+D(1+D)B_{\mu}\log(1+\frac{MI}{D(1+D)})}{MI} \Bigr).
\end{aligned}
\end{eqnarray*}
Because $M^2B_{\mu}>B_{\mu}$ for $M>1$, $\mathbb{E}[R_{super}(\hat{\bm{\theta}}_{super})-R_{super}(\bm{\theta}_{super}^*)]\geq \mathbb{E}[R_{single}(\hat{\bm{\theta}})-R_{single}(\bm{\theta}^*)]$. 
\end{proof}

This Lemma reflects the reason why the superposition of point processes is treated as a challenge in a lot of previous work. 
It means that if the event sequences are generated by a single Hawkes process, learning from the superposition of the sequences suffers a higher excess risk. 
What is worse, the more event sequences are superposed, the higher risk we will have in the learning result. 
In such a situation, we need to avoid the superposition of Hawkes processes. 

Fortunately, it is hard to describe real-world data using a single Hawkes process model. 
To suppress the risk of model misspecification, we can use multiple Hawkes processes with shared impact functions and different exogenous intensities to describe the data generated by an endogenously-stationary system with various exogenous fluctuations. 
In this situation, applying the superposition-based learning strategy can be efficient, especially in the following case:
\begin{lemma}\label{lm2}
For the Hawkes processes with complementary exogenous intensities, $i.e.$, $\{HP(\bm{\mu}^m, \bm{\Phi})\}_{m=1}^{M}$ and $\mbox{supp}(\bm{\mu}^m)\cap\mbox{supp}(\bm{\mu}^{m'})=\emptyset$ for all $m\neq m'$, the superposition-based learning strategy always provides us with benefits on efficiency, $i.e.$, $\mathbb{E}[R_{super}(\hat{\bm{\theta}}_{super})-R_{super}(\bm{\theta}_{super}^*)]\leq \mathbb{E}[R_{multi}(\hat{\bm{\theta}}_{multi})-R_{multi}(\bm{\theta}_{multi}^*)]$.
\end{lemma}
\begin{proof}
Here, the $\mbox{supp}(\cdot)$ returns to the set of the indices of nonzero elements. 
Because $\mbox{supp}(\bm{\mu}^m)\cap\mbox{supp}(\bm{\mu}^{m'})=\emptyset$ for all $m\neq m'$, we have $\|\sum_{m=1}^M\bm{\mu}^m\|_2^2 = \sum_{m=1}^M\|\bm{\mu}^m\|_2^2\leq MB_{\mu}=B_{\Sigma\mu}$. 
Plugging the upper bound into the condition (\ref{cond}), we have
\begin{eqnarray*}
\begin{aligned}
MB_{\mu}\leq & MB_{\mu} + D(M+D)B_{\mu}\log\Bigl(1+\frac{MI}{D(M+D)}\Bigr) \\
&- D(1+D)B_{\mu}\log\Bigl(1+\frac{MI}{D(1+D)}\Bigr).
\end{aligned}
\end{eqnarray*}
This inequality always holds because $D(M+D)\log(1+\frac{MI}{D(M+D)})\geq D(1+D)\log(1+\frac{MI}{D(1+D)})$ for $M>1$, $D\geq 1$ and $I\geq 1$.
\end{proof}

\subsection{The cost of the benefits}
It should be noted that when we apply the superposition-based learning strategy, we have to increase the computational complexity to obtain the tighter bound of excess risk. 
In particular, the matrix $\bm{X}$ in $R_{multi}$ contains the accumulation of integral decay kernels corresponding to $M$ event sequences with $I$ events per each.
The computational complexity for getting the $\bm{X}$ is $\mathcal{O}(MI^2)$. 
When we superpose the $M$ event sequences together, we obtain a denser superposed sequence with $MI$ events, so the computational complexity of the matrix $\bm{X}$ in $R_{super}$ is $\mathcal{O}(M^2I^2)$. 

Additionally, after applying the superposition-based learning strategy, we cannot learn the different exogenous intensities simultaneously with the impact functions because they have been accumulated by the superposition operation. 

\begin{figure*}[t]
\centering
\subfigure[Least squares, $D=5$]{
\includegraphics[width=0.24\linewidth]{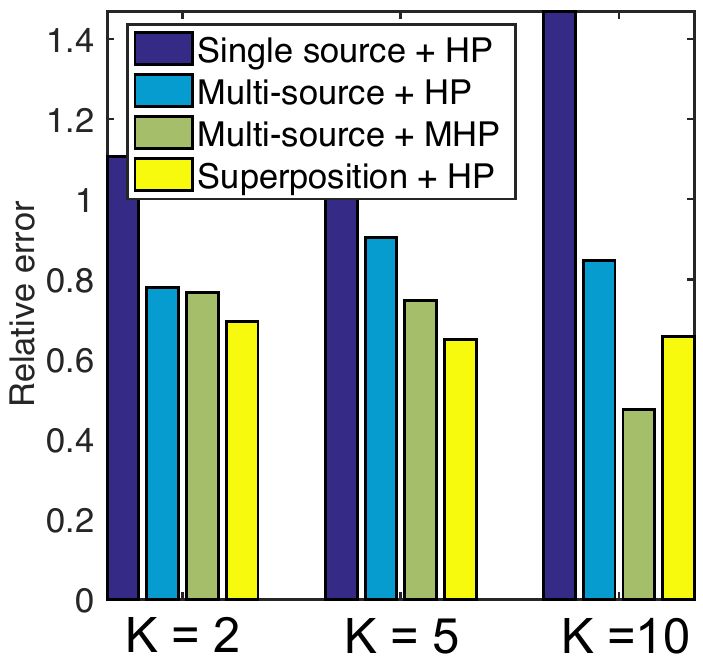}\label{fig:syn_ls1}}
\subfigure[Least squares, $D=10$]{
\includegraphics[width=0.24\linewidth]{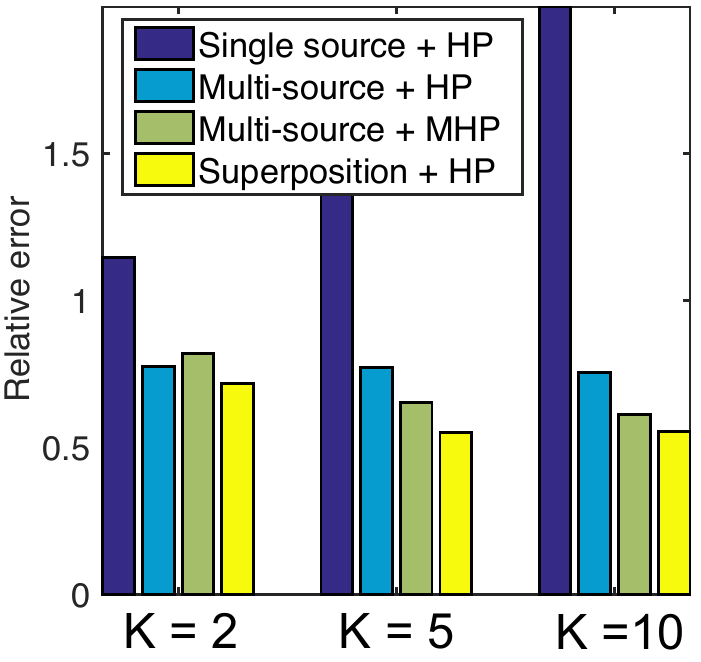}\label{fig:syn_ls2}}
\subfigure[MLE, $D=5$]{
\includegraphics[width=0.24\linewidth]{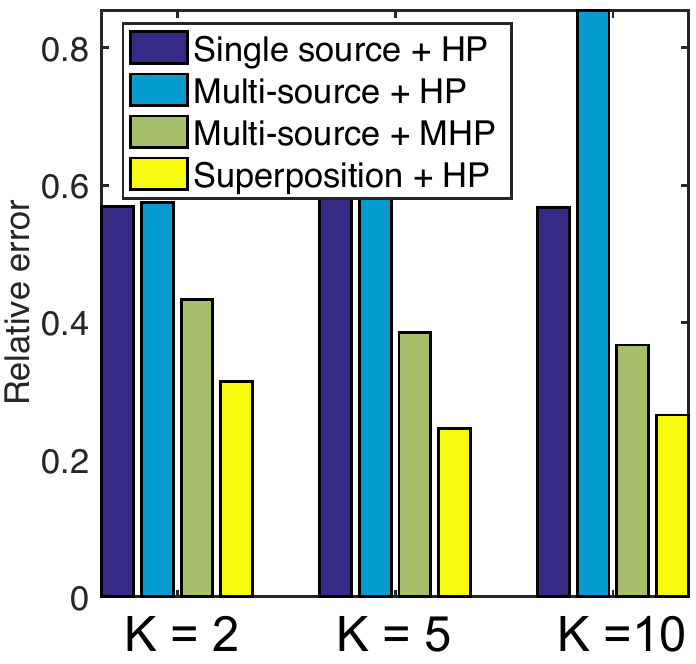}\label{fig:syn_mle1}}
\subfigure[MLE, $D=10$]{
\includegraphics[width=0.24\linewidth]{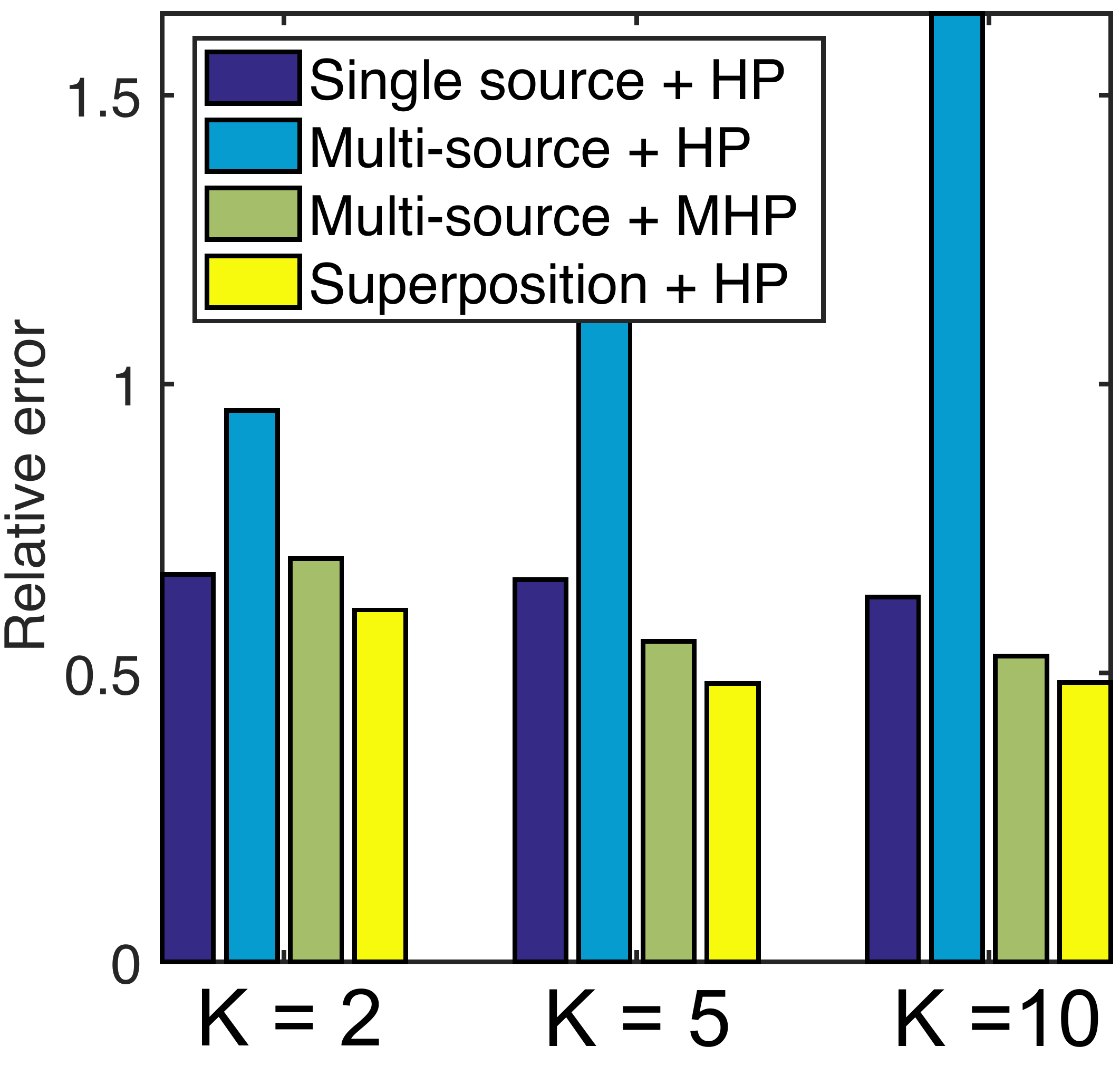}\label{fig:syn_mle2}}
\caption{Estimation errors of impact functions obtained by various methods.}
\end{figure*}

\section{Experiments}
\subsection{Validations based on synthetic data}
To verify the benefits from superposed Hawkes processes, we first test our superposition-based learning strategy on a synthetic data set and compare it with its competitors on learning errors of impact functions. 
The synthetic data set is generated as follows:
Given $K$ $D$-dimensional Hawkes process models, we generate $20$ event sequences for each. 
Each sequence has about $50$ events in the time window $[0,100]$. 
These Hawkes processes share the same impact functions, which are parameterized as an infectivity matrix $\bm{A}\in\mathbb{R}^{D\times D}$ and a predefined decay kernel $\exp(-t)$. 
The matrix $\bm{A}$ is random with spectral norm $\|\bm{A}\|_2=0.5$. 
The exogenous intensity $\bm{\mu}$ of each Hawkes process is a sparse vector, in which only one element is nonzero. 
The location of the nonzero element is randomly selected and the value is uniformly sampled from the interval $[0, 1]$. 
Given the parameters, all the event sequences are simulated by the branching process-based method in~\citep{moller2006approximate}. 
The number of models, $K$, is set to be $2$, $5$ and $10$, respectively. 
The dimension $D$ is set to be $5$ or $10$.

This synthetic data set and its simulation process imitate the behaviors in social networks. 
The users in the network correspond to the dimensions of Hawkes process and their relationships are captured by the infectivity matrix. 
The information sources in the network are different in different situations. 
Each event sequence records the behaviors of the users when one of them releases some information. 
The information releasing behaviors of the source user are modeled by a homogeneous Poisson process. 
The following behaviors of other users are triggered accordingly, which are modeled by inhomogeneous Poisson processes. 
The superposition of all these Poisson processes is a Hawkes process. 
More details can be found in~\citep{moller2006approximate,farajtabar2014shaping}.

We can learn the infectivity matrix based on the following four strategies:

\textbf{1. Single source + HP:} Select the event sequences corresponding to a specific kind of Hawkes process and learn a single Hawkes process.

\textbf{2. Multi-source + HP:} Ignore the diversity of the exogenous intensities and learn a single Hawkes process from all event sequences.

\textbf{3. Multi-source + MHP:} Consider the diversity of the exogenous intensities and learn multiple Hawkes processes from all event sequences.

\textbf{4. Proposed Superposition + HP:} Learn a single Hawkes process from the superposition of the event sequences.

The superposition is achieved as follows: $20$ superposed event sequences are obtained.  
Each of them is the superposition of $K$ event sequences and the $k$-th sequences is one of the $20$ sequences corresponding to the $k$-th Hawkes processes. 

Using the least-squares-based learning method mentioned in Section~\ref{ssec2:ls}, we learn the infectivity matrix based on the strategies above, respectively. 
Figure~\ref{fig:syn_ls1} and Figure~\ref{fig:syn_ls2} compare these strategies on the relative estimation error of the infectivity matrix ($i.e.$, $\frac{\|\widehat{\bm{A}}-\bm{A}^*\|_F}{\|\bm{A}^*\|_F}$) when $D=5$ and $10$. 
The results are the average of $10$ trials. 
We can find that the ``Single source + HP'' is the worst because it only takes advantage of the information of one source. 
Although it may avoid the problem of model misspecification, using much fewer samples makes it suffer from over-fitting and leads to large estimation errors. 
The ``Multi-source + HP'' strategy improves the learning results because it fully uses all event sequences. 
However, the improvement is limited because it ignores the diversity of the exogenous intensities and only learns a misspecified model. 
The ``Multi-source + MHP'' strategy is much better than the previous two strategies because it captures the diversity of the exogenous intensities by introducing more parameters. 
In the case with $K=10$ and $D=5$, it is even better than the proposed ``Superposition + HP'' strategy. 
Finally, the proposed ``Superposition + HP'' strategy achieves the smallest learning errors in most situations, which means that this strategy indeed obtains some benefits from the superposed Hawkes process, as shown in Lemma~\ref{lm2}. 

In our view, the reason for the inferiority of the proposed ``Superposition + HP'' strategy in the case with $K=10$ and $D=5$ is due to the lack of diversity of exogenous intensities. 
For $k=1,...,K$, the exogenous intensity vector of the $k$-th Hawkes process model has only one nonzero element. 
When $D<K$, there are some Hawkes process models having similar exogenous intensity vectors -- the locations of their nonzero elements are the same with each other. 
According to Lemma~\ref{lm1}, the superposition of these similar Hawkes processes does harm to the learning results. 
In other words, this failure case actually verifies our study results. 
When we increase the dimension of the model to $D=10$, we can find in Figure~\ref{fig:syn_ls2} that the performance of the proposed ``Superposition + HP'' strategy is consistently better than its competitors.

Another interesting phenomenon is that when we apply the maximum likelihood estimation (MLE) method~\citep{lewis2011nonparametric,zhou2013learning} to learn the Hawkes processes, the superposition-based strategy also outperforms to other strategies on the estimation error of the infectivity matrix. 
Figure~\ref{fig:syn_mle1} and Figure~\ref{fig:syn_mle2} visualize the comparisons. 
We can find that the proposed ``Superposition + HP'' strategy achieves much better learning results in all cases and the learning results obtained by MLE are better than those obtained by the least squares method. 
Additionally, different from the cases applying least squares method, when applying MLE, the ``Multi-source + HP'' is the worst strategy.
These observations reveal that 1) the MLE method has better sample complexity than the least squares method; 2) compared to over-fitting, the MLE method is more sensitive to the problem of model misspecification. 
The theoretical analysis of the benefits from superposed Hawkes processes in the framework of MLE is an interesting problem, which is left for our future work.

\subsection{Applications to the cold-start problem of recommendation systems}
A potential application of our work is solving the cold-start problem of recommendation systems. 
Specifically, the cold-start of a recommendation system means recommending certain items to the users having extremely few recorded behaviors, which is significant for practical recommendation systems. 
Traditional solutions of the cold-start problem rely on the side information like users' profiles, but they ignore the fact that for the users having few records the side information of them is likely to be limited as well. 
Therefore, how to solve the cold-start problem without side information is an important but challenging problem. 

Recently, Hawkes process-based recommendation systems have been proposed~\citep{du2015time}, which inspires us to solve the cold-start problem from the viewpoint of learning Hawkes processes. 
Generally, for each user her/his buying behavior of an item may trigger her/his following purchases of other items. 
The infectivity between different items is stationary, which is the endogenous nature of the recommendation system. 
The preference of each individual user corresponds to the personalized exogenous fluctuation of the system, and her/his purchases can be formulated as an event sequence, which can be captured by a Hawkes process. 
The event sequences of different users are instances of the Hawkes processes having the same impact functions (and infectivity matrix) but different exogenous intensities. 
If we can learn the infectivity matrix $\bm{A}$ of all $D$ items, we can recommend items for each user at time $t$ according to her/his historical buying behaviors $\mathcal{H}_t$ by finding the item with the highest endogenous intensity:
\begin{eqnarray*}
\begin{aligned}
d_{next}=\arg\max_{d\in\mathcal{D}}\sideset{}{_{(t_i,d_i)\in\mathcal{H}_t}}\sum a_{dd_i}\exp(-w(t-t_i)).
\end{aligned}
\end{eqnarray*}

In the cold-start problem, the event sequences are extremely short, so it is hard to learn a reliable infectivity matrix. 
With the help of the superposition-based learning strategy, we can increase the robustness of learned infectivity matrix and recommend items with higher accuracy. 

\begin{table}[t]
\small
\centering
\caption{Statistics of our data set.}\label{tab1}
\vspace{3pt}
\begin{tabular}{c|c|c|c}
\hline\hline
Category  &\#Users   &\#Items&\#Ratings  \\ \hline
{Baby}&1240&658&3142\\
{Garden}&650&466&1522\\
{Pet}&2128&958&5240\\ \hline\hline
\end{tabular}
\end{table}

\begin{table*}[t]
\small
\centering
\caption{Summary of the performance for various methods.}\label{tab2}
\vspace{3pt}
\setlength{\tabcolsep}{0.7pt}
\begin{tabular}
{c|c|@{\hspace{2pt}}c@{\hspace{4pt}}c@{\hspace{4pt}}c@{\hspace{3pt}}|@{\hspace{3pt}}c@{\hspace{4pt}}c@{\hspace{4pt}}c@{\hspace{3pt}}|@{\hspace{3pt}}c@{\hspace{4pt}}c@{\hspace{4pt}}c@{\hspace{3pt}}|@{\hspace{3pt}}c@{\hspace{4pt}}c@{\hspace{4pt}}c@{\hspace{3pt}}|@{\hspace{3pt}}c@{\hspace{4pt}}c@{\hspace{4pt}}c@{\hspace{1pt}}} 
\hline\hline
\multicolumn{2}{c|@{\hspace{2pt}}}{Method} &\multicolumn{3}{c|@{\hspace{3pt}}}{MostPopular}
&\multicolumn{3}{c|@{\hspace{3pt}}}{BPR} &\multicolumn{3}{c|@{\hspace{3pt}}}{FPMC}
&\multicolumn{3}{c|@{\hspace{3pt}}}{Multi-source+MHP} &\multicolumn{3}{c}{Superposition+HP}\\ \hline
\multicolumn{2}{c|@{\hspace{2pt}}}{Metric} &$P@N$ &$R@N$ &$F_1@N$
&$P@N$ &$R@N$ &$F_1@N$ &$P@N$ &$R@N$ &$F_1@N$
&$P@N$ &$R@N$ &$F_1@N$ &$P@N$ &$R@N$ &$F_1@N$\\ \hline
\multirow{3}{*}{Top5} &Baby &0.145&0.726&0.242 
                            &0.306&1.532&0.511
                            &\textbf{0.484}&\textbf{2.419}&\textbf{0.806}
                            &0.339&1.694&0.565
                            &0.306&1.532&0.511\\
                    &Garden &0.277&1.385&0.462
                            &0.646&3.231&1.077
                            &0.277&1.385&0.462
                            &0.739&3.692&1.231
                            &\textbf{1.046}&\textbf{5.231}&\textbf{1.744}\\
                    &Pet    &0.517&2.585&0.862    
                            &0.526&2.632&0.877
                            &0.517&2.585&0.862
                            &0.780&3.900&1.300
                            &\textbf{0.864}&\textbf{4.323}&\textbf{1.441}\\ \hline
\multirow{3}{*}{Top10}&Baby &0.234&2.339&0.425
                            &\textbf{0.379}&\textbf{3.790}&\textbf{0.689}
                            &0.307&3.065&0.557
                            &0.218&2.177&0.396
                            &0.282&2.822&0.513\\ 
                    &Garden &0.246&2.462&0.448
                            &0.431&4.308&0.783
                            &0.308&3.077&0.559
                            &0.646&6.461&1.174
                            &\textbf{0.800}&\textbf{8.000}&\textbf{1.454}\\
                    &Pet    &0.371&3.712&0.675
                            &0.428&4.276&0.778
                            &0.470&4.700&0.854 
                            &0.549&5.498&1.000    
                            &\textbf{0.630}&\textbf{6.297}&\textbf{1.145}\\ \hline
\multirow{3}{*}{Top20}&Baby &0.335&6.694&0.638                            
                            &0.294&5.887&0.561
                            &\textbf{0.339}&\textbf{6.774}&\textbf{0.645}
                            &0.194&3.871&0.369    
                            &0.254&5.081&0.484\\
                    &Garden &0.369&7.385&0.703
                            &0.431&8.615&0.821
                            &0.300&6.000&0.571
                            &0.439&8.769&0.835    
                            &\textbf{0.508}&\textbf{10.154}&\textbf{0.967}\\
                    &Pet    &0.374&7.472&0.712
                            &0.465&9.305&0.886
                            &0.371&7.425&0.707
                            &0.338&6.767&0.645    
                            &\textbf{0.489}&\textbf{9.774}&\textbf{0.931}\\ \hline\hline
\end{tabular}\label{tab:result}
\end{table*}

The training and testing data used in this work are from the Amazon product data set (APD) provided by~\citep{he2016ups}. 
The APD contains millions of buying-and-rating behaviors to the items grouped into $24$ categories. 
For each categories, millions of user-item pairs spanning from May 1996 to July 2014 and their time stamps are recorded. 
Focusing on the cold-start problem, we select the users having extremely fewer purchases to recommend items. 
Specifically, we preprocess the buying-and-rating behaviors of three categories ($i.e.$, ``Baby'', ``Patio, Lawn and Garden'', and ``Pet Supplies'') as follows. 
For the items having more than $40$ rating behaviors, we select their users satisfying three conditions: 1) the number of  behaviors of the users spanning from January 2014 to April 2014 is no more than $3$; 2) the scores they gave to these items are $4$ or $5$; 3) they bought and rated at least one item from April 2014 to July 2014. 
After the preprocessing above, we obtain a subset of the APD to train and test different methods in the cold-start situation. 
The statistics of our data set is given in Table~\ref{tab1}. 
According to users' buying-and-rating behaviors from January 2014 to April 2014, we aim to predict (recommend) items for them. 
Because during this period only one or two buying behaviors happened, this is a typical cold-start problem.

To demonstrate the usefulness of our strategy, we compare the ``Superposition + HP'' strategy with its most powerful competitor ``Multi-source + MHP''. 
Additionally, we consider three popular recommendation methods, including recommending the most popular item to all users (MostPopular), the Bayesian personalized ranking (BPR) in~\citep{rendle2009bpr} and the factorization of personalized Markov chains (FPMC) in~\citep{rendle2010factorizing}. 

We evaluate the recommendation results achieved by various methods and analyze their performance in the cold-start situation. 
For each method, we define the generated recommendation list for user $m$, $m=1,...,M$, as $\bm{r}^m = \{d_1^m, d_2^m, \cdots, d_N^M\}$, where $N$ is the number of recommended items, $d_i^m\in\mathcal{D}$ is ranked at the $i$-th position in $\bm{r}^m$. 
Suppose the real set of the items that user $m$ will buy is $\bm{t}^m$, we thus use the top-$N$ precision ($P@N$), recall ($R@N$) and $F_{1}$-score ($F_1@N$) as the measurements, which are defined as: 
\begin{eqnarray*}
\begin{aligned}
P@N  &= \frac{1}{M}\sum_{m} P_m@N =\frac{1}{M}\sum_{m} \frac{|\bm{r}^m \cap \bm{t}^m|}{|\bm{r}^m|}\times 100\%\\
R@N &= \frac{1}{M}\sum_{m} R_m@N  =\frac{1}{M}\sum_{m} \frac{|\bm{r}^m \cap \bm{t}^m|}{|\bm{t}^m|}\times 100\%\\
F_{1}@N &= \frac{1}{M}\sum_{m} F_{1m}@N  =\frac{1}{M}\sum_{m} \frac{2\cdot P_m@N \cdot R_m@N}{P_m@N+R_m@N}
\end{aligned}
\end{eqnarray*}
In this experiment, we set $N=5$, $10$ and $20$, respectively.

Table~\ref{tab2} summarizes the performance of various methods on the three categories listed in Table~\ref{tab1}. 
For the categories ``Garden'' and ``Pet'', applying the proposed ``Superposition + HP'' startegy improves the performance of recommendation systems in the cold-start situation. 
The gains obtained by our method on the three measurements are more than $10\%$ consistently compared with other methods. 
These results verify the potential of our method to solve the cold-start problem of recommendation systems. 

However, for the category ``Baby'', we can find that our method is inferior to the BPR or the FPMC method. 
According to our analysis, a possible reason for this phenomenon is that the buying-and-rating behaviors for the items of ``Baby'' category may not obey the Hawkes process model. 
Evidence supporting this explanation is that both ``Multi-source + MHP'' and ``Superposition + HP'' obtain unsatisfying recommendation results for this category. 
It should be noted that even if the results provided by the superposition-based learning strategy are not optimal, it still outperforms the ``Multi-source + MHP'' strategy, which means that the benefits from superposed Hawkes process for learning infectivity matrix are still available.

\textcolor{black}{The synthetic and real-world experiments mentioned above are implemented based on our MATLAB Hawkes process toolbox ``THAP''~\citep{xu2017thap}, which can be found at \url{https://github.com/HongtengXu/Hawkes-Process-Toolkit}. 
In particular, for each real-world data set involving about one thousand users and hundreds of items, the runtime of our method is less than $1$ minute and without any acceleration.}

\section{Conclusions and Future Work}
We have studied the properties of superposed Hawkes processes and have explored the potential benefits provided by the superposition operation for learning Hawkes process models. 
We demonstrate that with the help of superposition we can estimate the impact functions (or infectivity matrix) of the target Hawkes processes with lower excess risk in the framework of least squares-based learning. 
The typical feasible and infeasible conditions are given as well. 
We verify our theoretical results on synthetic data and show the potential of superposition-based learning strategy to solve the cold-start problem of recommendation systems. 

The experimental results in Figure~\ref{fig:syn_mle1} and Figure~\ref{fig:syn_mle2} imply that the benefits from superposed Hawkes processes are also available in the framework of maximum likelihood estimation. 
In the future, we plan to analyze the influence of superposition on the maximum likelihood estimation of Hawkes processes in theory. 
Additionally, studying the superposition of other kinds of point process models is also our future work.

\section*{Acknowledgments}
\textcolor{black}{
This work was supported in part by DARPA, DOE, NIH, NSF and ONR.
We also thank reviewers for their insightful and helpful comments.
}

\small
\bibliographystyle{icml2017}
\bibliography{example_paper}

\end{document}